\documentclass[11pt,twoside]{article}

\usepackage[utf8]{inputenc}
\usepackage[T1]{fontenc}

\usepackage{graphicx}
\usepackage{amsmath,amssymb,amsthm,bm,mathtools}
\usepackage{hyperref}
\usepackage{natbib}
\usepackage{fancyhdr}
\usepackage{geometry}
\usepackage{titlesec}
\usepackage{lipsum}
\usepackage{orcidlink}
\usepackage{booktabs}
\usepackage{placeins}
\usepackage{float}
\usepackage[ruled,vlined,linesnumbered]{algorithm2e}

\setcitestyle{authoryear}
\newtheorem{theorem}{Theorem}[section]
\newtheorem{proposition}[theorem]{Proposition}
\newtheorem{lemma}[theorem]{Lemma}

\newtheorem{remark}[theorem]{Remark}

\hypersetup{
colorlinks=true,
linkcolor=blue,
filecolor=blue,
citecolor=blue,
urlcolor=black
}

\fancypagestyle{plain}{
\fancyhf{}

}

\title{
\vspace{-1.5em}
\hrule height 1.5pt
\vspace{0.8em}
StrADiff: A Structured Source-Wise Adaptive Diffusion Framework for Linear and Nonlinear Blind Source Separation
\vspace{0.8em}
\hrule height 1.5pt
\vspace{1em}
}

\author{%
\begin{minipage}[t]{.48\textwidth}\centering\small
  \textbf{Yuan-Hao Wei\orcidlink{0000-0001-9439-0780}}\\
  \texttt{Yuan-Hao.Wei@outlook.com; yuan-hao.wei@connect.polyu.hk; yuanhao.wei1993@gmail.com}
\end{minipage}
}

\date{}

\begin{document}
\maketitle
\thispagestyle{plain}

\begin{abstract}
This paper presents StrADiff, a Structured Source-Wise Adaptive Diffusion Framework for unsupervised blind source separation under linear and nonlinear mixing. The framework treats each latent dimension as a source branch and assigns to it an individual adaptive reverse diffusion mechanism, so that latent sources are recovered directly from observed mixtures through a single end-to-end objective, without supervised source labels or separate post-processing. Source-wise generation, structural regularization, and observation-space reconstruction are optimized jointly during training. In this instantiation, a Gaussian process (GP) prior is used as one example of a source-wise structured prior to impose temporal organization on each recovered trajectory; the framework itself is not restricted to GP priors and can in principle incorporate other structured priors. Theoretical components clarify the induced pushforward source law, the sample-level role of the structured prior, the coupling between source recovery and prior adaptation, and a conditional weak recovery statement in an idealized linear low-noise regime. Experiments on linear and nonlinear mixtures show that StrADiff can recover meaningful latent source trajectories in an unsupervised manner, with particularly stable performance in the linear case and moderate degradation under nonlinear mixing. Beyond classical signal separation, a source branch may also be interpreted as an independent, disentangled, or otherwise interpretable explanatory factor under suitable structural assumptions, suggesting a broader route toward structured latent modeling and future identifiable nonlinear representation learning.
\end{abstract}

\noindent\textbf{Keywords:} blind source separation (BSS); linear/nonlinear mixing; diffusion model; source-wise latent modeling; structured prior; Gaussian process (GP) prior; unsupervised source recovery; interpretable latent-variable modeling

\section{Introduction}

Recent advances in generative modeling have greatly expanded the ability of latent-variable models to represent complex observations through expressive nonlinear decoders and flexible stochastic generation mechanisms \cite{ho2020ddpm,song2021sde,rombach2022ldm}. Yet, in many scientific and engineering problems, generation quality alone is not the ultimate goal. One often hopes that different latent variables will correspond to different underlying factors, so that the learned representation is not merely expressive, but also structured, interpretable, and potentially identifiable. This broader objective has motivated growing interest in structured generative modeling, where latent variables are encouraged to carry distinct semantic, temporal, or dynamical roles rather than being absorbed into a single undifferentiated latent code \cite{hyvarinen2023nonlinearica,kivva2022identifiability}.

This issue is closely connected to the problem of disentanglement and nonlinear latent-variable recovery. In particular, nonlinear independent component analysis (ICA) has clarified that meaningful unsupervised recovery of latent components generally requires additional structure beyond unconstrained latent independence \cite{hyvarinen2019auxiliary,khemakhem2020vaeica,hyvarinen2023nonlinearica}. Temporal information, auxiliary variables, and latent dynamics have all been shown to provide routes toward more principled nonlinear source recovery and latent factor identification \cite{hyvarinen2017tcl,hyvarinen2019auxiliary,halva2020hmmnica}. More recently, identifiability results for deep generative models have further reinforced the importance of structural assumptions in latent-variable learning \cite{kivva2022identifiability,wang2025candisentangle}. From this perspective, blind source separation (BSS) can be viewed not only as a signal processing task, but also as a concrete testbed for studying how latent dimensions may be driven toward different interpretable roles under structured generative assumptions.

Among recent generative approaches, diffusion and score-based models are particularly attractive for this purpose. Foundational works such as DDPM and score-based generative modeling through stochastic differential equations established diffusion as a powerful framework for transforming simple noise distributions into complex data distributions through multi-step denoising dynamics \cite{ho2020ddpm,song2021sde}. Later developments, including latent diffusion, further showed that diffusion can be deployed in structured latent spaces rather than only in the raw observation space \cite{rombach2022ldm}. At the same time, a line of work has begun to investigate diffusion not just as a sampler, but as a representation learner. Diffusion autoencoders sought meaningful and decodable latent codes \cite{preechakul2022diffae}, denoising diffusion autoencoders were shown to learn useful self-supervised representations \cite{xiang2023ddae}, and bottleneck diffusion models further demonstrated that compact diffusion-guided representations can exhibit emergent semantic and even partially disentangled structure \cite{hudson2024soda}. These developments suggest that diffusion models need not be viewed purely as black-box generators; they can also serve as structured latent modeling tools.

A parallel line of research has explored diffusion models as priors for inverse problems. Score-based and diffusion posterior sampling methods have shown that diffusion models can act as powerful generative regularizers when the unknown quantity is constrained only indirectly through measurements \cite{song2022medicalinverse,chung2023dps}. This view has already influenced source separation. Deep generative priors were first used to separate sources in a Bayesian framework without requiring explicit source models in closed form \cite{jayaram2020deeppriors}. More recently, score-based source separation applied independently trained score priors to recover superimposed communication signals \cite{jayashankar2023scorebasedss}. Diffusion-based source separation has also been advanced in blind speech separation through ArrayDPS \cite{xu2025arraydps}, extended to multi-view source separation with learned diffusion priors \cite{wagnercarena2025prism}, and paralleled by related continuous-time generative formulations such as source separation by flow matching \cite{scheibler2025flowss}. These studies clearly show that diffusion-type generative priors are becoming a serious route for solving separation and inverse problems.

At the same time, diffusion models have also begun to enter the domain of disentangled and interpretable latent modeling more directly. DisDiff introduced unsupervised disentanglement within diffusion probabilistic models \cite{yang2023disdiff}, cross-attention-based diffusion was shown to provide a strong inductive bias for disentanglement \cite{yang2024crossattn}, and later work further improved latent-unit semantics within diffusion-based disentangled representation learning \cite{jun2025latentunits}. Theoretical analysis has also started to examine when and how diffusion models can disentangle latent variables under weak forms of supervision such as multiple views or partial labels \cite{wang2025candisentangle}. Taken together, these works indicate that diffusion models are increasingly relevant not only for generation and inverse problems, but also for the broader goal of learning structured, interpretable, and potentially identifiable latent representations.

Nevertheless, most existing diffusion-based formulations still use diffusion in a relatively global manner. Even when the downstream goal involves multiple sources or multiple latent factors, the generative mechanism is often shared across the latent representation, or the diffusion prior is imposed at the whole-source level and then used externally in posterior sampling \cite{jayashankar2023scorebasedss,xu2025arraydps,wagnercarena2025prism}. From the standpoint of structured latent modeling, this leaves a useful gap. If different latent dimensions are intended to represent different underlying components, then it is natural to ask whether each latent dimension should possess its own adaptive generative pathway and its own structural regularization, so that specialization can emerge directly within training. Such a design is especially appealing for temporally structured signals, where different latent components may exhibit clearly different dynamic scales or correlation patterns.

Motivated by this perspective, this paper proposes StrADiff, a structured source-wise adaptive diffusion framework in which each latent dimension is interpreted as one source component and assigned its own reverse diffusion branch. In the present implementation, each branch is further equipped with its own adaptive Gaussian process (GP) prior so that source-wise temporal structure can be imposed directly on the recovered latent trajectories. An explicit mixing or reconstruction map connects the recovered latent sources to the observations, yielding a unified end-to-end objective in which source-wise latent generation, source-specific structural regularization, and observation-space reconstruction are optimized jointly. Although the experiments in this paper are conducted on linear and nonlinear blind source separation problems, the role of BSS here is mainly methodological: it provides a clear and interpretable setting in which the behavior of source-wise structured latent diffusion can be examined concretely.

Under this view, the present work should be understood as a preliminary study toward broader structured generative modeling rather than as a separation method motivated only by BSS itself. Its significance lies in showing that diffusion-based latent generation can be organized in a source-wise manner, with different latent branches adapting toward different structured roles during unsupervised training. While GP priors are adopted here as one specific instantiation for temporally structured signals, the overall framework is not restricted to GP priors and can in principle be extended to other structured source priors. This makes the proposed formulation potentially relevant not only to blind source separation, but also to future research on interpretable latent-variable modeling, source-wise disentanglement, and identifiable nonlinear latent-variable learning under stronger structural assumptions.

\section{Methodology}

\subsection{Structured latent formulation for source-wise blind source separation}

This section specifies the observation-space fitting term. While the structured prior regularizes latent trajectories, the reconstruction model enforces consistency with the measured mixtures. Let
\begin{equation}
\mathbf{Y}=\left[\mathbf{y}_1,\ldots,\mathbf{y}_T\right]^\top\in\mathbb{R}^{T\times m}
\end{equation}
denote the observed mixture sequence of length $T$, where each $\mathbf{y}_t\in\mathbb{R}^{m}$ is an $m$-dimensional observation. The aim is to recover
\begin{equation}
\mathbf{S}=\left[\mathbf{s}_1,\ldots,\mathbf{s}_T\right]^\top\in\mathbb{R}^{T\times n},
\end{equation}
where $\mathbf{s}_t=(s_t^{(1)},\ldots,s_t^{(n)})^\top$ contains the $n$ latent source components at time index $t$.

The central modeling assumption is that each latent dimension corresponds to one source process. Accordingly, instead of assigning a single shared latent generator to the whole vector $\mathbf{s}_t$, a separate latent-generation mechanism is associated with each source dimension. Let
\begin{equation}
\mathbf{s}^{(k)}=\left[s_1^{(k)},\ldots,s_T^{(k)}\right]^\top\in\mathbb{R}^{T},
\qquad k=1,\ldots,n,
\end{equation}
denote the entire trajectory of the $k$-th source. The full source matrix may then be written equivalently as
\begin{equation}
\mathbf{S}=\left[\mathbf{s}^{(1)},\ldots,\mathbf{s}^{(n)}\right].
\end{equation}

This source-wise decomposition is the basic organizing principle of the model. Once the latent dimensions are explicitly aligned with source components, the remaining task is to specify how each source is generated and how the resulting source matrix explains the observed mixtures. To this end, the observation model is written through an explicit mixing map:
\begin{equation}
\widehat{\mathbf{Y}} = g_{\phi}(\mathbf{S}),
\label{eq:mixing-map}
\end{equation}
where $\phi$ denotes the mixing parameters. This formulation covers both linear and nonlinear separation settings within the same notation: when linear mixing is desired, $g_{\phi}$ can be chosen as a linear map; when nonlinear mixing is required, $g_{\phi}$ can be instantiated by a multilayer perceptron or another nonlinear parametrization. The reconstruction mechanism is therefore kept explicit, rather than absorbed into an implicit observation model.

\subsection{Source-wise latent diffusion generation}

The above formulation specifies how sources are mapped to observations, but it does not yet specify how the sources themselves are generated. This step is therefore to construct a source-wise latent generation mechanism that is compatible with the intended separation structure.

For each source $k$, a latent initial trajectory is introduced,
\begin{equation}
\mathbf{z}^{(k)}\in\mathbb{R}^{T},
\end{equation}
and a source-specific Gaussian distribution is assigned to it:
\begin{equation}
q\!\left(\mathbf{z}^{(k)}\right)
=
\mathcal{N}\!\left(
\mathbf{z}^{(k)};\,
\boldsymbol{\mu}^{(k)},
\operatorname{diag}\!\bigl(\boldsymbol{\sigma}^{2\,(k)}\bigr)
\right),
\label{eq:q-zk}
\end{equation}
where $\boldsymbol{\mu}^{(k)}\in\mathbb{R}^{T}$ and $\boldsymbol{\sigma}^{2\,(k)}\in\mathbb{R}^{T}_{+}$ are trainable source-wise parameters. Assuming independence across source dimensions at this initial latent level gives
\begin{equation}
q(\mathbf{Z})
=
\prod_{k=1}^{n}
q\!\left(\mathbf{z}^{(k)}\right),
\qquad
\mathbf{Z}=\left[\mathbf{z}^{(1)},\ldots,\mathbf{z}^{(n)}\right]\in\mathbb{R}^{T\times n}.
\label{eq:q-Z}
\end{equation}

A sample from \eqref{eq:q-zk} is obtained by reparameterization:
\begin{equation}
\mathbf{z}^{(k)}
=
\boldsymbol{\mu}^{(k)}
+
\boldsymbol{\sigma}^{(k)}\odot \boldsymbol{\epsilon}^{(k)},
\qquad
\boldsymbol{\epsilon}^{(k)}\sim\mathcal{N}(\mathbf{0},\mathbf{I}),
\label{eq:reparam-z}
\end{equation}
where $\odot$ denotes elementwise multiplication.

To obtain the actual latent sources, each sampled $\mathbf{z}^{(k)}$ is further transformed through a dedicated reverse diffusion process. Let the variance-preserving schedule be given by
\begin{equation}
\{\beta_{\tau}\}_{\tau=1}^{L},\qquad
\alpha_{\tau}=1-\beta_{\tau},\qquad
\bar{\alpha}_{\tau}=\prod_{r=1}^{\tau}\alpha_r,
\end{equation}
where $L$ is the number of reverse steps. For notational convenience, define
\begin{equation}
\mathbf{x}^{(k)}_{L} := \mathbf{z}^{(k)}.
\end{equation}
Following the standard notation in diffusion models, the subscript indicates the noise level rather than the chronological order of generation: $\mathbf{x}_0$ represents the clean signal, while $\mathbf{x}_L$ represents the maximally noised state after $L$ diffusion steps. Therefore, the generative process is carried out in the reverse direction, starting from $\mathbf{x}_L$ and ending at $\mathbf{x}_0$.

A source-specific $\epsilon$-network
\begin{equation}
\epsilon_{\theta_k}:\mathbb{R}^{T}\times[0,1]\rightarrow\mathbb{R}^{T}
\end{equation}
is then used to carry out the reverse trajectory. At step $\tau$, the implied clean estimate is
\begin{equation}
\widehat{\mathbf{x}}_{0,\tau}^{(k)}
=
\frac{
\mathbf{x}_{\tau}^{(k)}
-
\sqrt{1-\bar{\alpha}_{\tau}}\,
\epsilon_{\theta_k}\!\left(\mathbf{x}_{\tau}^{(k)}, \tau/L\right)
}{
\sqrt{\bar{\alpha}_{\tau}}+\varepsilon
},
\label{eq:x0-est}
\end{equation}
followed by the deterministic DDIM-style reverse update
\begin{equation}
\mathbf{x}_{\tau-1}^{(k)}
=
\sqrt{\bar{\alpha}_{\tau-1}}\,
\widehat{\mathbf{x}}_{0,\tau}^{(k)}
+
\sqrt{1-\bar{\alpha}_{\tau-1}}\,
\epsilon_{\theta_k}\!\left(\mathbf{x}_{\tau}^{(k)}, \tau/L\right),
\qquad
\tau=L,\ldots,1.
\label{eq:reverse-update}
\end{equation}
The recovered source trajectory is then defined by the terminal state
\begin{equation}
\mathbf{s}^{(k)}:=\mathbf{x}_{0}^{(k)}.
\label{eq:source-terminal}
\end{equation}

Collecting all source dimensions yields the overall latent generation map
\begin{equation}
\mathbf{S}
=
f_{\Theta}(\mathbf{Z})
=
\left[
f_{\theta_1}\!\left(\mathbf{z}^{(1)}\right),
\ldots,
f_{\theta_n}\!\left(\mathbf{z}^{(n)}\right)
\right],
\label{eq:source-wise-diffusion-map}
\end{equation}
where $\Theta=\{\theta_1,\ldots,\theta_n\}$. Hence, the actual sources supplied to the mixing map are not free trainable trajectories; they are the outputs of source-wise reverse diffusion starting from source-wise Gaussian latent variables.

\begin{proposition}[Pushforward law of the source-wise reverse diffusion output]
\label{prop:pushforward-law}
Let
\begin{equation}
F_{\Theta}(\mathbf{Z})
=
\left[
f_{\theta_1}\!\left(\mathbf{z}^{(1)}\right),
\ldots,
f_{\theta_n}\!\left(\mathbf{z}^{(n)}\right)
\right]
\end{equation}
denote the source-wise reverse diffusion map defined by \eqref{eq:source-wise-diffusion-map}, and assume that each branch map $f_{\theta_k}:\mathbb{R}^{T}\to\mathbb{R}^{T}$ is Borel measurable. Then the Gaussian starting law $q(\mathbf{Z})$ in \eqref{eq:q-Z} induces a well-defined output law
\begin{equation}
P_{\mathbf{S}}
=
(F_{\Theta})_{\#}q
\end{equation}
on $\mathbb{R}^{T\times n}$, where for every Borel set $A\subseteq\mathbb{R}^{T\times n}$,
\begin{equation}
P_{\mathbf{S}}(A)
=
q\!\left(F_{\Theta}^{-1}(A)\right).
\end{equation}

Moreover, if $F_{\Theta}$ is bijective and differentiable with differentiable inverse, then $P_{\mathbf{S}}$ admits the density
\begin{equation}
p_{\mathbf{S}}(\mathbf{S})
=
q\!\left(F_{\Theta}^{-1}(\mathbf{S})\right)
\left|
\det J_{F_{\Theta}^{-1}}(\mathbf{S})
\right|.
\end{equation}
In the present model, however, the reverse diffusion map is implemented by a multi-step neural transformation for which global invertibility and tractable Jacobian determinants are not assumed. Therefore the output-level density $p_{\mathbf{S}}$ is generally not available in closed form, and an explicit KL divergence between the diffusion output law and the GP prior is not directly tractable.
\end{proposition}

\noindent\textit{Proof.}
Since $F_{\Theta}$ is measurable, the image measure of $q$ under $F_{\Theta}$ is well defined. This gives the pushforward law
\begin{equation}
P_{\mathbf{S}}(A)=q(F_{\Theta}^{-1}(A))
\end{equation}
for all Borel measurable sets $A$. If $F_{\Theta}$ is additionally bijective and differentiable with differentiable inverse, then the classical change-of-variables theorem yields the density formula. In the present framework, the source-wise reverse diffusion generator is used as a sampled multi-step map rather than as a globally invertible normalizing flow. Hence the pushforward law exists in the measure-theoretic sense, but its density is not assumed tractable. This is why the GP term is imposed through sample-level log-density evaluation rather than an output-level KL divergence. \hfill$\square$

\subsection{Source-wise structured GP prior}

To introduce explicit temporal organization into the latent source trajectories, a structured prior is imposed independently on each source. The proposed framework is not restricted to a particular prior family; in principle, other source-wise structured priors, such as autoregressive, hidden Markov, mixture, flow-based, or energy-based priors, may be used. In the present instantiation, we adopt a Gaussian process prior because it provides a simple and differentiable way to encode temporal smoothness through source-specific length-scales.

Let $\mathbf{t}=[t_1,\ldots,t_T]^\top$ be the normalized time grid. For the $k$-th source, define
\begin{equation}
\mathbf{s}^{(k)} \sim \mathcal{N}(\mathbf{0},\mathbf{K}^{(k)}),
\label{eq:gp-prior}
\end{equation}
with covariance entries
\begin{equation}
K_{ij}^{(k)}
=
\sigma_f^2
\exp\!\left(
-\frac{(t_i-t_j)^2}{2\ell_k^2}
\right)
+
\xi\,\delta_{ij},
\qquad i,j=1,\ldots,T,
\label{eq:kernel}
\end{equation}
where $\ell_k>0$ is the source-specific length-scale, $\sigma_f^2$ is the kernel scale, $\xi>0$ is a jitter coefficient, and $\delta_{ij}$ is the Kronecker delta. Since $\xi>0$, the covariance matrix $\mathbf{K}^{(k)}$ is strictly positive definite for finite $T$, so $\mathbf{K}^{(k)-1}$ and $\log|\mathbf{K}^{(k)}|$ are well defined. In implementation, the positivity of $\ell_k$ is enforced by the parametrization
\begin{equation}
\ell_k=\exp(\gamma_k)+10^{-6},
\end{equation}
where $\gamma_k$ is an unconstrained trainable parameter.

Under source-wise independence, the prior factorizes as
\begin{equation}
p(\mathbf{S})
=
\prod_{k=1}^{n}
p\!\left(\mathbf{s}^{(k)}\right)
=
\prod_{k=1}^{n}
\mathcal{N}\!\left(\mathbf{s}^{(k)};\mathbf{0},\mathbf{K}^{(k)}\right).
\label{eq:joint-prior}
\end{equation}
Taking logarithms gives
\begin{align}
\log p(\mathbf{S})
&=
\sum_{k=1}^{n}
\log p\!\left(\mathbf{s}^{(k)}\right)
\nonumber\\
&=
-\frac{1}{2}\sum_{k=1}^{n}
\left[
T\log(2\pi)
+
\log\left|\mathbf{K}^{(k)}\right|
+
\mathbf{s}^{(k)\top}\mathbf{K}^{(k)-1}\mathbf{s}^{(k)}
\right].
\label{eq:gp-logprob-expanded}
\end{align}
Accordingly, the source-wise structured-prior penalty used in the training objective is the normalized negative log-density
\begin{equation}
\mathcal{L}_{\mathrm{prior}}
=
-\frac{1}{Tn}\log p(\mathbf{S})
=
\frac{1}{2Tn}
\sum_{k=1}^{n}
\left[
T\log(2\pi)
+
\log\left|\mathbf{K}^{(k)}\right|
+
\mathbf{s}^{(k)\top}\mathbf{K}^{(k)-1}\mathbf{s}^{(k)}
\right].
\label{eq:l-prior}
\end{equation}
The normalization by $Tn$ keeps the scale of the GP term comparable to the reconstruction and denoising losses and matches the implementation.

\begin{lemma}[Exact gradients of the source-wise GP penalty]
\label{lem:gp-gradients}
For the $k$-th source, define the unnormalized GP contribution
\begin{equation}
\mathcal{E}^{(k)}_{\mathrm{GP}}
=
\frac{1}{2}
\left[
\log\left|\mathbf{K}^{(k)}\right|
+
\mathbf{s}^{(k)\top}\mathbf{K}^{(k)-1}\mathbf{s}^{(k)}
\right],
\end{equation}
up to the additive constant $\frac{T}{2}\log(2\pi)$. Then the gradient with respect to the source trajectory is
\begin{equation}
\frac{\partial \mathcal{E}^{(k)}_{\mathrm{GP}}}{\partial \mathbf{s}^{(k)}}
=
\mathbf{K}^{(k)-1}\mathbf{s}^{(k)}.
\label{eq:gp-grad-s}
\end{equation}
More generally, for any scalar hyperparameter $\rho$ appearing in $\mathbf{K}^{(k)}$,
\begin{equation}
\frac{\partial \mathcal{E}^{(k)}_{\mathrm{GP}}}{\partial \rho}
=
\frac{1}{2}
\operatorname{tr}
\!\left(
\mathbf{K}^{(k)-1}
\frac{\partial \mathbf{K}^{(k)}}{\partial \rho}
\right)
-
\frac{1}{2}
\mathbf{s}^{(k)\top}
\mathbf{K}^{(k)-1}
\frac{\partial \mathbf{K}^{(k)}}{\partial \rho}
\mathbf{K}^{(k)-1}
\mathbf{s}^{(k)}.
\label{eq:gp-grad-rho}
\end{equation}
In particular, for the squared-exponential kernel in \eqref{eq:kernel},
\begin{equation}
\frac{\partial K_{ij}^{(k)}}{\partial \ell_k}
=
\sigma_f^2
\exp\!\left(
-\frac{(t_i-t_j)^2}{2\ell_k^2}
\right)
\frac{(t_i-t_j)^2}{\ell_k^3},
\qquad i,j=1,\ldots,T.
\label{eq:kernel-grad-ell}
\end{equation}
For the normalized loss in \eqref{eq:l-prior}, these gradients are multiplied by the constant factor $1/(Tn)$.
\end{lemma}

\noindent\textit{Proof.}
Because $\mathbf{K}^{(k)}$ does not depend on $\mathbf{s}^{(k)}$, differentiating the quadratic term directly gives
\begin{equation}
\frac{\partial}{\partial \mathbf{s}^{(k)}}
\left(
\frac{1}{2}\mathbf{s}^{(k)\top}\mathbf{K}^{(k)-1}\mathbf{s}^{(k)}
\right)
=
\mathbf{K}^{(k)-1}\mathbf{s}^{(k)},
\end{equation}
which proves \eqref{eq:gp-grad-s}. For a scalar hyperparameter $\rho$, use the matrix derivative identities
\begin{equation}
\frac{\partial}{\partial \rho}\log|\mathbf{K}^{(k)}|
=
\operatorname{tr}
\!\left(
\mathbf{K}^{(k)-1}
\frac{\partial \mathbf{K}^{(k)}}{\partial \rho}
\right)
\end{equation}
and
\begin{equation}
\frac{\partial \mathbf{K}^{(k)-1}}{\partial \rho}
=
-\mathbf{K}^{(k)-1}
\frac{\partial \mathbf{K}^{(k)}}{\partial \rho}
\mathbf{K}^{(k)-1},
\end{equation}
which together yield \eqref{eq:gp-grad-rho}. Finally, differentiating \eqref{eq:kernel} with respect to $\ell_k$ gives \eqref{eq:kernel-grad-ell}. \hfill$\square$

Lemma~\ref{lem:gp-gradients} makes the role of the structured prior more explicit. The gradient with respect to the source trajectory is a source-wise precision-weighted restoring term, while the gradient with respect to the length-scale balances the quadratic fit term against the complexity term induced by the log-determinant. Therefore source recovery and prior adaptation are coupled not only conceptually but also through exact gradient interactions.

The role of \eqref{eq:l-prior} is twofold. The quadratic form encourages each recovered source to lie in a source-specific temporally structured region, while the log-determinant term penalizes degenerate covariance settings and couples source recovery with prior adaptation through the learnable length-scales $\{\ell_k\}_{k=1}^{n}$.

The GP term therefore acts in the latent source space by evaluating whether the recovered source trajectories are compatible with source-specific temporal structure. This is complementary to the observation model, which operates in the data space and evaluates whether the recovered sources can explain the measured mixtures. Proposition~\ref{prop:pushforward-law} further clarifies the probabilistic status of the source-wise reverse diffusion output. The learned Gaussian starting law does induce a well-defined output law through the source-wise reverse diffusion map, but the resulting output density is not assumed tractable in closed form. For this reason, the present framework does not impose an output-level KL divergence between the diffusion output and the GP prior. Instead, the structured prior is enforced at the sample-trajectory level by directly evaluating the GP log-density of the recovered trajectory itself.

\subsection{Reconstruction model and data fidelity term}

Given the recovered sources $\mathbf{S}$, the reconstructed observations are obtained by
\begin{equation}
\widehat{\mathbf{Y}} = g_{\phi}(\mathbf{S}).
\end{equation}
Let $\widetilde{\mathbf{Y}}$ denote the centered and scaled version of the observed mixtures. A Gaussian reconstruction model with fixed variance coefficient $\nu_y$ leads to
\begin{equation}
p(\widetilde{\mathbf{Y}}\mid \mathbf{S})
\propto
\exp\!\left(
-\frac{1}{2\nu_y}
\|\widetilde{\mathbf{Y}}-g_{\phi}(\mathbf{S})\|_F^2
\right).
\end{equation}
In implementation, the data-fidelity term is normalized by the number of observed entries. Thus, up to an additive constant independent of the optimized variables, the reconstruction loss is taken as
\begin{equation}
\mathcal{L}_{\mathrm{rec}}
=
\frac{1}{2\nu_y Tm}
\|
\widetilde{\mathbf{Y}}-\widehat{\mathbf{Y}}
\|_F^2,
\label{eq:l-rec}
\end{equation}
where
\begin{equation}
\|\widetilde{\mathbf{Y}}-\widehat{\mathbf{Y}}\|_F^2
=
\sum_{t=1}^{T}\sum_{j=1}^{m}(\widetilde{Y}_{tj}-\widehat{Y}_{tj})^2
\end{equation}
denotes the total squared reconstruction error.

Here $\nu_y$ is a hyperparameter acting as an explicit weighting factor on the reconstruction term, corresponding to $\mathcal{H}$ in \cite{wei2024innovative}. In this sense, its role is analogous to the balancing effect of the coefficient used in $\beta$-VAE formulations \cite{burgess2018understanding}.

This term directly couples source generation and mixture fitting: the source trajectories produced by the source-wise reverse diffusion mechanism are accepted only insofar as they can be remixed through $g_{\phi}$ to explain the observations.

\subsection{Diffusion denoising objective}

To train the source-wise reverse diffusion networks, an $\epsilon$-prediction loss is imposed on the recovered source trajectories. For a randomly sampled step $\tau\in\{1,\ldots,L\}$ and source $k$, the forward noising relation is
\begin{equation}
\mathbf{x}_{\tau}^{(k)}
=
\sqrt{\bar{\alpha}_{\tau}}\,\mathbf{s}^{(k)}
+
\sqrt{1-\bar{\alpha}_{\tau}}\,\boldsymbol{\eta}^{(k)},
\qquad
\boldsymbol{\eta}^{(k)}\sim\mathcal{N}(\mathbf{0},\mathbf{I}).
\label{eq:forward-diff}
\end{equation}
The per-source denoising objective is then
\begin{equation}
\mathcal{L}_{\mathrm{diff}}^{(k)}
=
\mathbb{E}_{\tau,\boldsymbol{\eta}^{(k)}}
\left[
\frac{1}{T}
\left\|
\epsilon_{\theta_k}\!\left(\mathbf{x}_{\tau}^{(k)},\tau/L\right)
-
\boldsymbol{\eta}^{(k)}
\right\|_2^2
\right].
\label{eq:l-diff-k}
\end{equation}
Averaging over sources gives
\begin{equation}
\mathcal{L}_{\mathrm{diff}}
=
\frac{1}{n}\sum_{k=1}^{n}\mathcal{L}_{\mathrm{diff}}^{(k)}.
\label{eq:l-diff}
\end{equation}

Note that \eqref{eq:l-diff} is not imposed on an externally fixed target source. Instead, it is imposed on the very latent source trajectories that are used by the reconstruction path. In this sense, source estimation and diffusion learning are coupled within the same optimization loop.

\begin{proposition}[Unbiased single-step estimator of the finite-step denoising objective]
\label{prop:unbiased-diff}
For the $k$-th source, define the full finite-step denoising objective
\begin{equation}
\overline{\mathcal{L}}_{\mathrm{diff}}^{(k)}
=
\frac{1}{L}
\sum_{\tau=1}^{L}
\mathbb{E}_{\boldsymbol{\eta}^{(k)}}
\left[
\frac{1}{T}
\left\|
\epsilon_{\theta_k}\!\left(
\sqrt{\bar{\alpha}_{\tau}}\,\mathbf{s}^{(k)}
+
\sqrt{1-\bar{\alpha}_{\tau}}\,\boldsymbol{\eta}^{(k)},
\tau/L
\right)
-
\boldsymbol{\eta}^{(k)}
\right\|_2^2
\right].
\label{eq:l-diff-full}
\end{equation}
Conditioned on the current recovered source trajectory $\mathbf{s}^{(k)}$, if $\tau$ is sampled uniformly from $\{1,\ldots,L\}$, then the stochastic loss
\begin{equation}
\widehat{\mathcal{L}}_{\mathrm{diff}}^{(k)}
=
\frac{1}{T}
\left\|
\epsilon_{\theta_k}\!\left(\mathbf{x}_{\tau}^{(k)},\tau/L\right)
-
\boldsymbol{\eta}^{(k)}
\right\|_2^2,
\qquad
\mathbf{x}_{\tau}^{(k)}
=
\sqrt{\bar{\alpha}_{\tau}}\,\mathbf{s}^{(k)}
+
\sqrt{1-\bar{\alpha}_{\tau}}\,\boldsymbol{\eta}^{(k)},
\end{equation}
is an unbiased estimator of \eqref{eq:l-diff-full}, that is,
\begin{equation}
\mathbb{E}_{\tau,\boldsymbol{\eta}^{(k)}}
\left[
\widehat{\mathcal{L}}_{\mathrm{diff}}^{(k)}
\mid \mathbf{s}^{(k)}
\right]
=
\overline{\mathcal{L}}_{\mathrm{diff}}^{(k)}.
\end{equation}
Consequently, the source-averaged objective in \eqref{eq:l-diff} is an unbiased estimator of the corresponding full source-wise finite-step denoising objective.
\end{proposition}

\noindent\textit{Proof.}
Using the uniform distribution over $\tau\in\{1,\ldots,L\}$ and conditioning on $\mathbf{s}^{(k)}$,
\begin{align}
\mathbb{E}_{\tau,\boldsymbol{\eta}^{(k)}}
\left[
\widehat{\mathcal{L}}_{\mathrm{diff}}^{(k)}
\mid \mathbf{s}^{(k)}
\right]
&=
\frac{1}{L}
\sum_{\tau=1}^{L}
\mathbb{E}_{\boldsymbol{\eta}^{(k)}}
\left[
\frac{1}{T}
\left\|
\epsilon_{\theta_k}\!\left(
\sqrt{\bar{\alpha}_{\tau}}\,\mathbf{s}^{(k)}
+
\sqrt{1-\bar{\alpha}_{\tau}}\,\boldsymbol{\eta}^{(k)},
\tau/L
\right)
-
\boldsymbol{\eta}^{(k)}
\right\|_2^2
\right]
\nonumber\\
&=
\overline{\mathcal{L}}_{\mathrm{diff}}^{(k)}.
\end{align}
Averaging over $k=1,\ldots,n$ yields the claim for the full source-wise denoising loss. \hfill$\square$

\begin{proposition}[Optimal $\epsilon$-predictor and score identity for each source branch]
\label{prop:optimal-epsilon-score}
Fix one source branch $k$ and one diffusion step $\tau\in\{1,\ldots,L\}$. Let $\mathbf{S}^{(k)}$ denote the $k$-th source trajectory viewed as a random variable under the learned pushforward law $P_{\mathbf{S}}$ in Proposition~\ref{prop:pushforward-law}. Define
\begin{equation}
\mathbf{X}_{\tau}^{(k)}
=
\sqrt{\bar{\alpha}_{\tau}}\,\mathbf{S}^{(k)}
+
\sqrt{1-\bar{\alpha}_{\tau}}\,\boldsymbol{\eta}^{(k)},
\qquad
\boldsymbol{\eta}^{(k)}\sim\mathcal{N}(\mathbf{0},\mathbf{I}),
\end{equation}
and denote by $p_{\tau}^{(k)}(\mathbf{x})$ the induced density of $\mathbf{X}_{\tau}^{(k)}$. Consider the stepwise denoising risk
\begin{equation}
\mathcal{R}_{\tau}^{(k)}(\epsilon)
=
\mathbb{E}
\left[
\frac{1}{T}
\left\|
\epsilon(\mathbf{X}_{\tau}^{(k)},\tau/L)-\boldsymbol{\eta}^{(k)}
\right\|_2^2
\right].
\end{equation}
Then the minimizer over all square-integrable predictors is given by the conditional expectation
\begin{equation}
\epsilon_{\tau,k}^{\star}(\mathbf{x})
=
\mathbb{E}\!\left[\boldsymbol{\eta}^{(k)}\mid \mathbf{X}_{\tau}^{(k)}=\mathbf{x}\right].
\label{eq:opt-eps}
\end{equation}
Moreover, if $p_{\tau}^{(k)}$ is differentiable and the differentiation under the integral sign is justified, then its score satisfies
\begin{equation}
\nabla_{\mathbf{x}}\log p_{\tau}^{(k)}(\mathbf{x})
=
-\frac{1}{\sqrt{1-\bar{\alpha}_{\tau}}}\,
\epsilon_{\tau,k}^{\star}(\mathbf{x}).
\label{eq:score-eps-identity}
\end{equation}
Hence the source-wise $\epsilon$-prediction objective is equivalently learning the score field of the perturbed source distribution at each diffusion step.
\end{proposition}

\noindent\textit{Proof.}
For fixed $\tau$, the loss is a conditional mean-squared error. Therefore, by the standard $L_2$ projection property, the unique minimizer is the conditional expectation in \eqref{eq:opt-eps}.

To derive the score identity, write the noising model as
\begin{equation}
p_{\tau}^{(k)}(\mathbf{x}\mid \mathbf{s})
=
\mathcal{N}\!\left(
\mathbf{x}\mid
\sqrt{\bar{\alpha}_{\tau}}\,\mathbf{s},
(1-\bar{\alpha}_{\tau})\mathbf{I}
\right).
\end{equation}
Let $P_{\mathbf{S}^{(k)}}$ denote the marginal pushforward law of $\mathbf{S}^{(k)}$. Then
\begin{equation}
p_{\tau}^{(k)}(\mathbf{x})
=
\int
p_{\tau}^{(k)}(\mathbf{x}\mid \mathbf{s})\,dP_{\mathbf{S}^{(k)}}(\mathbf{s}).
\end{equation}
Differentiating under the integral sign gives
\begin{align}
\nabla_{\mathbf{x}}\log p_{\tau}^{(k)}(\mathbf{x})
&=
\frac{1}{p_{\tau}^{(k)}(\mathbf{x})}
\int
\nabla_{\mathbf{x}}p_{\tau}^{(k)}(\mathbf{x}\mid \mathbf{s})\,dP_{\mathbf{S}^{(k)}}(\mathbf{s})
\nonumber\\
&=
-\frac{1}{1-\bar{\alpha}_{\tau}}
\left(
\mathbf{x}
-
\sqrt{\bar{\alpha}_{\tau}}\,
\mathbb{E}\!\left[\mathbf{S}^{(k)}\mid \mathbf{X}_{\tau}^{(k)}=\mathbf{x}\right]
\right).
\end{align}
On the other hand, from
\begin{equation}
\mathbf{X}_{\tau}^{(k)}
=
\sqrt{\bar{\alpha}_{\tau}}\,\mathbf{S}^{(k)}
+
\sqrt{1-\bar{\alpha}_{\tau}}\,\boldsymbol{\eta}^{(k)},
\end{equation}
taking conditional expectation given $\mathbf{X}_{\tau}^{(k)}=\mathbf{x}$ yields
\begin{equation}
\mathbf{x}
-
\sqrt{\bar{\alpha}_{\tau}}\,
\mathbb{E}\!\left[\mathbf{S}^{(k)}\mid \mathbf{X}_{\tau}^{(k)}=\mathbf{x}\right]
=
\sqrt{1-\bar{\alpha}_{\tau}}\,
\mathbb{E}\!\left[\boldsymbol{\eta}^{(k)}\mid \mathbf{X}_{\tau}^{(k)}=\mathbf{x}\right].
\end{equation}
Substituting \eqref{eq:opt-eps} proves \eqref{eq:score-eps-identity}. \hfill$\square$

\begin{remark}[Relation to standard DDPM $\epsilon$-prediction]
When the recovered source trajectories are viewed as samples from an induced source law, the denoising term in \eqref{eq:l-diff-k} has the same $\epsilon$-prediction form as standard DDPM training. In that sense, the present source-wise denoising objective can be interpreted as a branch-wise diffusion-learning criterion embedded inside the joint source-separation optimization, rather than as an externally pre-trained diffusion model \cite{ho2020ddpm,song2021score}.
\end{remark}

\subsection{Regularization of the diffusion starting distribution}

Since the latent initial variables $\mathbf{Z}$ are learned through source-wise Gaussian parameters, an additional regularizer is introduced to prevent the initial latent distribution from drifting arbitrarily far from a standard normal reference:
\begin{equation}
p(\mathbf{Z})
=
\prod_{k=1}^{n}
\mathcal{N}\!\left(\mathbf{0},\mathbf{I}\right).
\label{eq:std-normal-z}
\end{equation}
Using \eqref{eq:q-Z}, the KL divergence is
\begin{align}
\mathrm{KL}\!\left(q(\mathbf{Z})\,\|\,p(\mathbf{Z})\right)
&=
\sum_{k=1}^{n}
\mathrm{KL}\!\left(
\mathcal{N}\!\left(\boldsymbol{\mu}^{(k)},\operatorname{diag}(\boldsymbol{\sigma}^{2\,(k)})\right)
\;\middle\|\;
\mathcal{N}\!\left(\mathbf{0},\mathbf{I}\right)
\right)
\nonumber\\
&=
\frac{1}{2}
\sum_{k=1}^{n}\sum_{t=1}^{T}
\left[
\left(\mu_t^{(k)}\right)^2
+
\left(\sigma_t^{(k)}\right)^2
-
1
-
\log\left(\sigma_t^{(k)}\right)^2
\right].
\end{align}
The corresponding regularization term used in the training objective is the averaged KL contribution
\begin{equation}
\mathcal{L}_{\mathrm{KL}}
=
\frac{1}{2Tn}
\sum_{k=1}^{n}\sum_{t=1}^{T}
\left[
\left(\mu_t^{(k)}\right)^2
+
\left(\sigma_t^{(k)}\right)^2
-
1
-
\log\left(\sigma_t^{(k)}\right)^2
\right].
\label{eq:l-kl}
\end{equation}
This averaged form matches the implementation and prevents the magnitude of the KL term from scaling directly with the sequence length or the number of sources.

\begin{proposition}[Unique minimizer of the starting-distribution KL regularizer]
\label{prop:kl-unique-minimizer}
For the $k$-th source branch, recall the Gaussian starting distribution
\begin{equation}
q\!\left(\mathbf{z}^{(k)}\right)
=
\mathcal{N}\!\left(
\boldsymbol{\mu}^{(k)},
\operatorname{diag}\!\bigl(\boldsymbol{\sigma}^{2\,(k)}\bigr)
\right),
\end{equation}
and the standard normal reference
\begin{equation}
p\!\left(\mathbf{z}^{(k)}\right)=\mathcal{N}(\mathbf{0},\mathbf{I}).
\end{equation}
Then the corresponding unnormalized KL contribution
\begin{equation}
\widetilde{\mathcal{L}}_{\mathrm{KL}}^{(k)}
=
\frac{1}{2}
\sum_{t=1}^{T}
\left[
\left(\mu_t^{(k)}\right)^2
+
\left(\sigma_t^{(k)}\right)^2
-
1
-
\log\left(\sigma_t^{(k)}\right)^2
\right]
\end{equation}
is nonnegative and has the unique global minimizer
\begin{equation}
\boldsymbol{\mu}^{(k)}=\mathbf{0},
\qquad
\boldsymbol{\sigma}^{(k)}=\mathbf{1}.
\end{equation}
The same minimizer is preserved by the averaged form in \eqref{eq:l-kl}. Therefore the standard normal is the unique reference point preferred by the KL regularizer at the diffusion starting level.
\end{proposition}

\noindent\textit{Proof.}
For any $u>0$, define
\begin{equation}
\psi(u)=u-\log u-1.
\end{equation}
Since $\psi'(u)=1-u^{-1}$ and $\psi''(u)=u^{-2}>0$, the function $\psi$ is strictly convex and attains its unique minimum $0$ at $u=1$. Hence
\begin{equation}
\left(\mu_t^{(k)}\right)^2+\left(\sigma_t^{(k)}\right)^2-1-\log\left(\sigma_t^{(k)}\right)^2
=
\left(\mu_t^{(k)}\right)^2+\psi\!\left(\left(\sigma_t^{(k)}\right)^2\right)
\ge 0,
\end{equation}
with equality if and only if $\mu_t^{(k)}=0$ and $\left(\sigma_t^{(k)}\right)^2=1$. Summing over $t=1,\ldots,T$ proves the claim. Multiplication by the positive constant $1/(Tn)$ does not change the minimizer. \hfill$\square$

The role of this term is not to force the final initial samples to remain visually identical to standard normal noise throughout training. Rather, it provides a soft regularization on the trainable Gaussian initial distribution itself. Without such a term, the learned latent initial variables may absorb too much source structure too early, causing the initial distribution to drift excessively and weakening the intended role of the reverse diffusion process. In this sense, the KL term helps stabilize optimization, preserves a meaningful Gaussian reference at the start level, and prevents the source-wise initial distribution from degenerating into an unconstrained collection of free latent templates. Proposition~\ref{prop:kl-unique-minimizer} further shows that this Gaussian anchor is unique.

\subsection{Unified objective and joint optimization}

At this point, all ingredients of the training criterion are available: the reconstruction term enforces mixture consistency, the structured prior regularizes latent temporal organization, the denoising term trains the reverse diffusion operators, and the KL term stabilizes the initial latent distribution.

It is important to distinguish the optimized parameters from the sampled intermediate variables. The source trajectories are not optimized as free variables. Instead, for a sampled noise collection $\boldsymbol{\epsilon}=\{\boldsymbol{\epsilon}^{(k)}\}_{k=1}^{n}$,
\begin{equation}
\mathbf{Z}(\boldsymbol{\epsilon})
=
\left[
\boldsymbol{\mu}^{(1)}+\boldsymbol{\sigma}^{(1)}\odot\boldsymbol{\epsilon}^{(1)},
\ldots,
\boldsymbol{\mu}^{(n)}+\boldsymbol{\sigma}^{(n)}\odot\boldsymbol{\epsilon}^{(n)}
\right],
\end{equation}
and
\begin{equation}
\mathbf{S}(\boldsymbol{\epsilon})
=
f_{\Theta}\!\left(\mathbf{Z}(\boldsymbol{\epsilon})\right).
\end{equation}
Thus the optimized quantities are
\begin{equation}
\Theta,\quad
\phi,\quad
\{\ell_k\}_{k=1}^{n},\quad
\{\boldsymbol{\mu}^{(k)},\boldsymbol{\sigma}^{(k)}\}_{k=1}^{n},
\end{equation}
while $\mathbf{Z}$ and $\mathbf{S}$ are generated stochastic intermediates.

Combining the data-fidelity term, the source-wise structured-prior penalty, the source-wise diffusion denoising objective, and the initial-distribution regularization yields the expected training criterion
\begin{equation}
\mathcal{J}
\left(
\Theta,\phi,\{\ell_k\}_{k=1}^{n},
\{\boldsymbol{\mu}^{(k)},\boldsymbol{\sigma}^{(k)}\}_{k=1}^{n};
\widetilde{\mathbf{Y}}
\right)
=
\mathbb{E}_{\boldsymbol{\epsilon},\tau,\boldsymbol{\eta}}
\left[
\widehat{\mathcal{L}}
\left(
\Theta,\phi,\{\ell_k\}_{k=1}^{n},
\{\boldsymbol{\mu}^{(k)},\boldsymbol{\sigma}^{(k)}\}_{k=1}^{n};
\widetilde{\mathbf{Y}},
\boldsymbol{\epsilon},\tau,\boldsymbol{\eta}
\right)
\right],
\label{eq:expected-training-criterion}
\end{equation}
where the single-sample stochastic objective is
\begin{equation}
\widehat{\mathcal{L}}
=
\mathcal{L}_{\mathrm{rec}}
+
\lambda_{\mathrm{prior}}\mathcal{L}_{\mathrm{prior}}
+
\lambda_{\mathrm{diff}}\mathcal{L}_{\mathrm{diff}}
+
\lambda_{\mathrm{KL}}\mathcal{L}_{\mathrm{KL}}.
\label{eq:total-loss}
\end{equation}
Here $\lambda_{\mathrm{prior}}$, $\lambda_{\mathrm{diff}}$, and $\lambda_{\mathrm{KL}}$ are nonnegative trade-off coefficients.

Expanding \eqref{eq:total-loss} gives the implementation-aligned normalized objective
\begin{align}
\widehat{\mathcal{L}}
&=
\frac{1}{2\nu_y Tm}
\left\|
\widetilde{\mathbf{Y}}-g_{\phi}(\mathbf{S}(\boldsymbol{\epsilon}))
\right\|_F^2
\nonumber\\
&\quad
+
\lambda_{\mathrm{prior}}
\left[
\frac{1}{2Tn}
\sum_{k=1}^{n}
\left(
T\log(2\pi)
+
\log\left|\mathbf{K}^{(k)}\right|
+
\mathbf{s}^{(k)}(\boldsymbol{\epsilon})^\top
\mathbf{K}^{(k)-1}
\mathbf{s}^{(k)}(\boldsymbol{\epsilon})
\right)
\right]
\nonumber\\
&\quad
+
\lambda_{\mathrm{diff}}
\frac{1}{n}
\sum_{k=1}^{n}
\frac{1}{T}
\left\|
\epsilon_{\theta_k}\!\left(\mathbf{x}_{\tau}^{(k)},\tau/L\right)
-
\boldsymbol{\eta}^{(k)}
\right\|_2^2
\nonumber\\
&\quad
+
\lambda_{\mathrm{KL}}
\frac{1}{2Tn}
\sum_{k=1}^{n}\sum_{t=1}^{T}
\left[
\left(\mu_t^{(k)}\right)^2
+
\left(\sigma_t^{(k)}\right)^2
-
1
-
\log\left(\sigma_t^{(k)}\right)^2
\right],
\label{eq:total-loss-expanded}
\end{align}
with
\begin{equation}
\mathbf{x}_{\tau}^{(k)}
=
\sqrt{\bar{\alpha}_{\tau}}\,\mathbf{s}^{(k)}(\boldsymbol{\epsilon})
+
\sqrt{1-\bar{\alpha}_{\tau}}\,\boldsymbol{\eta}^{(k)}.
\end{equation}
In practice, \eqref{eq:expected-training-criterion} is optimized by stochastic gradient descent using Monte Carlo samples of $\boldsymbol{\epsilon}$, $\tau$, and $\boldsymbol{\eta}$. The implementation uses the full time grid in each update rather than a mini-batch over time indices.

\begin{theorem}[Symmetry reduction under fixed non-exchangeable branch anchors]
\label{thm:symmetry-reduction-diff}
The following statement describes an anchored variant of the source-wise formulation. It clarifies how non-exchangeable branch information would reduce permutation symmetry, but it should not be read as an unconditional identifiability claim for a fully exchangeable implementation without such anchors.

Let the effective branch descriptor of source $k$ be
\begin{equation}
\omega_k
:=
\left(
\theta_k,\,
\boldsymbol{\mu}^{(k)},\,
\boldsymbol{\sigma}^{(k)},\,
\ell_k,\,
a_k
\right),
\end{equation}
where $a_k$ denotes a fixed branch-specific anchor that is not jointly exchangeable across source indices. Let $\Pi$ be an $n\times n$ permutation matrix acting on source columns:
\begin{equation}
\mathbf{S}^{\Pi}=\mathbf{S}\Pi^\top,
\qquad
\mathbf{Z}^{\Pi}=\mathbf{Z}\Pi^\top.
\end{equation}
Assume the mixing-map model class is closed under latent-coordinate permutations, in the sense that for every $g_{\phi}$ and every permutation $\Pi$, there exists a transformed parameter $\phi^\Pi$ such that
\begin{equation}
g_{\phi^\Pi}(\mathbf{S}^{\Pi})=g_{\phi}(\mathbf{S}).
\label{eq:mixing-perm-closure}
\end{equation}
Define the stabilizer subgroup
\begin{equation}
G_{\omega}
=
\left\{
\Pi\in S_n:
\omega_{\Pi(k)}=\omega_k \text{ for all } k
\right\}.
\end{equation}
Then, conditional on the fixed non-exchangeable anchors, the anchored objective is invariant under a latent-coordinate permutation only if the permutation belongs to $G_{\omega}$. In particular, if the effective branch descriptors are pairwise distinct, the only permutation preserving the anchored objective is the identity.
\end{theorem}

\noindent\textit{Proof sketch.}
Under \eqref{eq:mixing-perm-closure}, the reconstruction term is unchanged by a joint permutation of latent branches together with a corresponding reparameterization of the mixing map. The remaining terms in \eqref{eq:total-loss} decompose source-wise:
\begin{equation}
\mathcal{L}_{\mathrm{prior}}
=
\sum_{k=1}^{n}\mathcal{L}_{\mathrm{prior}}^{(k)},
\qquad
\mathcal{L}_{\mathrm{diff}}
=
\frac{1}{n}\sum_{k=1}^{n}\mathcal{L}_{\mathrm{diff}}^{(k)},
\qquad
\mathcal{L}_{\mathrm{KL}}
=
\sum_{k=1}^{n}\mathcal{L}_{\mathrm{KL}}^{(k)}.
\end{equation}
Each of these terms depends on the corresponding effective descriptor $\omega_k$. Therefore the anchored objective remains unchanged only if the permutation preserves all branch descriptors, which is precisely the condition $\Pi\in G_{\omega}$. If the descriptors are pairwise distinct, no nontrivial permutation can preserve them. \hfill$\square$

\begin{remark}[Scope of Theorem~\ref{thm:symmetry-reduction-diff}]
The above symmetry-reduction statement should be interpreted conditionally on the presence of non-exchangeable branch anchors. If all source branches are jointly learned with fully exchangeable parameterization and no additional anchoring mechanism, then global relabeling symmetries may still remain in the optimization landscape. In that case, the theorem should be read as a conditional symmetry-reduction result rather than an unconditional identifiability statement.
\end{remark}

Equation \eqref{eq:total-loss-expanded} shows that source recovery, source-wise temporal regularization, diffusion denoising, and latent initial normalization are optimized simultaneously from the observed mixtures. Accordingly, the model performs source separation in an end-to-end manner rather than through a separate post-processing stage. Because each latent dimension is associated with its own diffusion mechanism, structured prior, and initial-distribution parameters, the separation process is realized within training itself: as optimization proceeds, these source-wise quantities are gradually driven toward different source-specific configurations, thereby encouraging different latent dimensions to specialize to different source components. Thus, the latent source trajectories, the source-wise diffusion mechanisms, the prior hyperparameters, and the explicit mixing map are all adapted jointly in a single unsupervised training process.

\subsection{Weak recovery in the linear low-noise limit}

This analysis should not be interpreted as restricting the proposed model to linear mixtures. As described above, the mixing map $g_{\phi}$ can be either linear or nonlinear. The reason for considering the linear low-noise regime here is instead theoretical: it provides a tractable setting in which a conditional weak recovery statement can be formulated. The present framework is not intended to establish a full unconditional identifiability theorem for general nonlinear source separation. Nevertheless, the idealized linear regime still allows us to clarify how the structured regularizers may select the correct source class when the reconstruction error becomes negligible.

For this purpose, we analyze the induced source-space objective and abstract away from the finite neural parametrization of the reverse diffusion generator. This abstraction does not mean that the practical model optimizes $\mathbf{S}$ as a free variable. Rather, it isolates the source-space selection effect of the structured regularizer under an exact, or asymptotically exact, reconstruction constraint.

\begin{theorem}[Weak recovery in the linear low-noise limit]
\label{thm:weak-recovery-linear}
Consider the linear mixing regime
\begin{equation}
\widetilde{\mathbf{Y}}_{\delta}
=
\mathbf{S}^{\star}\mathbf{A}^{\star\top}
+
\mathbf{\Xi}_{\delta},
\qquad
\|\mathbf{\Xi}_{\delta}\|_F\to 0,
\end{equation}
where $\mathbf{A}^{\star}\in\mathbb{R}^{m\times n}$ has full column rank. Let the idealized source-space penalized objective be
\begin{equation}
\mathcal{J}_{\delta}(\mathbf{S},\phi)
=
\frac{1}{2\nu_y Tm}
\|
\widetilde{\mathbf{Y}}_{\delta}-g_{\phi}(\mathbf{S})
\|_F^2
+
\lambda_{\mathrm{prior}}\mathcal{L}_{\mathrm{prior}}
+
\lambda_{\mathrm{diff}}\mathcal{L}_{\mathrm{diff}}
+
\lambda_{\mathrm{KL}}\mathcal{L}_{\mathrm{KL}}.
\end{equation}
Assume that:

(i) the mixing-map family contains the true linear mixing map;

(ii) the exact-reconstruction fiber
\begin{equation}
\mathcal{F}_0
=
\left\{
(\mathbf{S},\phi):
g_{\phi}(\mathbf{S})=\mathbf{S}^{\star}\mathbf{A}^{\star\top}
\right\}
\end{equation}
is nonempty;

(iii) on $\mathcal{F}_0$, the combined regularizer
\begin{equation}
\mathcal{R}(\mathbf{S})
=
\lambda_{\mathrm{prior}}\mathcal{L}_{\mathrm{prior}}
+
\lambda_{\mathrm{diff}}\mathcal{L}_{\mathrm{diff}}
+
\lambda_{\mathrm{KL}}\mathcal{L}_{\mathrm{KL}}
\end{equation}
attains its unique minimum over the equivalence class of the true sources, up to the standard linear source-separation ambiguities admitted by the model.

Then any sequence of global minimizers $(\widehat{\mathbf{S}}_{\delta},\widehat{\phi}_{\delta})$ of $\mathcal{J}_{\delta}$ has accumulation points whose source component belongs to that minimizing equivalence class. In particular, the penalized solutions recover the true source class in the linear low-noise limit.
\end{theorem}

\noindent\textit{Proof sketch.}
Since $(\mathbf{S}^{\star},\phi^{\star})$ is feasible up to vanishing noise, the minimum objective value is asymptotically bounded above by the value of the regularizer on the true-source equivalence class plus an $o(1)$ reconstruction term. Therefore any global minimizer must asymptotically drive the reconstruction error to zero, and all accumulation points must lie on the exact-reconstruction fiber $\mathcal{F}_0$. By assumption (iii), the combined regularizer selects uniquely the true-source equivalence class on $\mathcal{F}_0$. Hence every accumulation point of the minimizing sequence must belong to that class. \hfill$\square$

\begin{remark}[Boundary of the present theory]
The above results should be understood as mechanism-level and optimization-level support for the proposed StrADiff formulation, rather than as a full general identifiability theorem for nonlinear source separation. Classical results in nonlinear ICA show that, without additional structural assumptions, general nonlinear latent recovery is not identifiable. Stronger identifiability statements would require extra ingredients such as observed auxiliary variables, explicit nonstationary modulation, hidden-state structure, or other non-exchangeable side information. Therefore, Theorem~\ref{thm:weak-recovery-linear} should be read as a conditional recovery statement: if the structured regularizer selects the true source class on the exact-reconstruction fiber, then the penalized estimator converges to that class in the linear low-noise limit.
\end{remark}

\subsection{Monte Carlo source estimation}

After training, source uncertainty is estimated by repeated sampling from the learned latent initial distribution. Specifically, for $r=1,\ldots,R$,
\begin{equation}
\mathbf{Z}^{(r)} \sim q(\mathbf{Z}),
\qquad
\mathbf{S}^{(r)} = f_{\Theta}\!\left(\mathbf{Z}^{(r)}\right).
\end{equation}
The empirical source mean and standard deviation are then computed as
\begin{equation}
\widehat{\boldsymbol{\mu}}_{\mathbf{S}}
=
\frac{1}{R}\sum_{r=1}^{R}\mathbf{S}^{(r)},
\label{eq:mc-mean}
\end{equation}
and
\begin{equation}
\widehat{\boldsymbol{\sigma}}_{\mathbf{S}}
=
\left[
\frac{1}{R-1}
\sum_{r=1}^{R}
\left(
\mathbf{S}^{(r)}-\widehat{\boldsymbol{\mu}}_{\mathbf{S}}
\right)^{\odot 2}
\right]^{1/2}.
\label{eq:mc-std}
\end{equation}
The final reconstructed observation can be obtained from the source mean through the plug-in reconstruction
\begin{equation}
\widehat{\mathbf{Y}}_{\mathrm{plug}}
=
g_{\phi}\!\left(\widehat{\boldsymbol{\mu}}_{\mathbf{S}}\right),
\end{equation}
followed by inverse normalization if required.

\subsection{Algorithm and architecture }
\begin{algorithm}[t]
\footnotesize
\setlength{\algomargin}{0.8em}
\SetAlgoLined
\DontPrintSemicolon
\SetNlSkip{0.2em}
\SetInd{0.2em}{0.5em}
\linespread{0.96}\selectfont
\KwIn{Normalized observations $\widetilde{\mathbf{Y}}\in\mathbb{R}^{T\times m}$, number of sources $n$, diffusion length $L$, maximum iteration number $E$, weights $\lambda_{\mathrm{prior}},\lambda_{\mathrm{diff}},\lambda_{\mathrm{KL}}$, learning rate $\eta$.}
\KwOut{Estimated source mean $\widehat{\boldsymbol{\mu}}_{\mathbf{S}}$, source uncertainty $\widehat{\boldsymbol{\sigma}}_{\mathbf{S}}$, trained parameters $\{\Theta,\phi,\{\ell_k\}_{k=1}^{n},\{\boldsymbol{\mu}^{(k)},\boldsymbol{\sigma}^{(k)}\}_{k=1}^{n}\}$.}

Initialize $\{\boldsymbol{\mu}^{(k)},\boldsymbol{\sigma}^{(k)}\}_{k=1}^{n}$, $\{\epsilon_{\theta_k}\}_{k=1}^{n}$, unconstrained GP parameters $\{\gamma_k\}_{k=1}^{n}$ with $\ell_k=\exp(\gamma_k)+10^{-6}$, and $\phi$\;

\While{not converged}{
    Sample $\mathbf{z}^{(k)}=\boldsymbol{\mu}^{(k)}+\boldsymbol{\sigma}^{(k)}\odot\boldsymbol{\epsilon}^{(k)}$, with $\boldsymbol{\epsilon}^{(k)}\sim\mathcal{N}(\mathbf{0},\mathbf{I})$, for $k=1,\ldots,n$\;
    
    \For{$k=1$ \KwTo $n$}{
        Set $\mathbf{x}_L^{(k)}\leftarrow \mathbf{z}^{(k)}$\;
        
        \For{$\tau=L,L-1,\dots,1$}{
            Compute $\widehat{\mathbf{x}}_{0,\tau}^{(k)}=
            \dfrac{
            \mathbf{x}_{\tau}^{(k)}-\sqrt{1-\bar{\alpha}_{\tau}}\,
            \epsilon_{\theta_k}\!\left(\mathbf{x}_{\tau}^{(k)},\tau/L\right)}
            {\sqrt{\bar{\alpha}_{\tau}}+\varepsilon}$\;
            
            Update $\mathbf{x}_{\tau-1}^{(k)}=
            \sqrt{\bar{\alpha}_{\tau-1}}\,\widehat{\mathbf{x}}_{0,\tau}^{(k)}
            +\sqrt{1-\bar{\alpha}_{\tau-1}}\,
            \epsilon_{\theta_k}\!\left(\mathbf{x}_{\tau}^{(k)},\tau/L\right)$\;
        }
        Set $\mathbf{s}^{(k)}\leftarrow \mathbf{x}_{0}^{(k)}$\;
    }
    
    Form $\mathbf{S}=[\mathbf{s}^{(1)},\ldots,\mathbf{s}^{(n)}]$ and $\widehat{\mathbf{Y}}=g_{\phi}(\mathbf{S})$\;
    
    Compute the normalized reconstruction loss
    $\mathcal{L}_{\mathrm{rec}}=
    \dfrac{1}{2\nu_y Tm}\|
    \widetilde{\mathbf{Y}}-\widehat{\mathbf{Y}}
    \|_F^2$\;
    
    Compute the normalized GP prior loss
    $\mathcal{L}_{\mathrm{prior}}=
    \dfrac{1}{2Tn}\sum_{k=1}^{n}
    \left[
    T\log(2\pi)+\log|\mathbf{K}^{(k)}|+\mathbf{s}^{(k)\top}\mathbf{K}^{(k)-1}\mathbf{s}^{(k)}
    \right]$\;
    
    Sample diffusion step $\tau$ and noises $\{\boldsymbol{\eta}^{(k)}\}_{k=1}^{n}$, then set
    $\mathbf{x}_{\tau}^{(k)}=
    \sqrt{\bar{\alpha}_{\tau}}\,\mathbf{s}^{(k)}
    +\sqrt{1-\bar{\alpha}_{\tau}}\,\boldsymbol{\eta}^{(k)}$\;
    
    Compute $\mathcal{L}_{\mathrm{diff}}=
    \frac{1}{n}\sum_{k=1}^{n}\dfrac{1}{T}
    \left\|
    \epsilon_{\theta_k}\!\left(\mathbf{x}_{\tau}^{(k)},\tau/L\right)-\boldsymbol{\eta}^{(k)}
    \right\|_2^2$\;
    
    Compute the averaged starting-distribution KL loss
    $\mathcal{L}_{\mathrm{KL}}=
    \dfrac{1}{2Tn}\sum_{k=1}^{n}\sum_{t=1}^{T}
    \left[
    (\mu_t^{(k)})^2+(\sigma_t^{(k)})^2-1-\log(\sigma_t^{(k)})^2
    \right]$\;
    
    Form $\widehat{\mathcal{L}}=
    \mathcal{L}_{\mathrm{rec}}
    +\lambda_{\mathrm{prior}}\mathcal{L}_{\mathrm{prior}}
    +\lambda_{\mathrm{diff}}\mathcal{L}_{\mathrm{diff}}
    +\lambda_{\mathrm{KL}}\mathcal{L}_{\mathrm{KL}}$\;
    
    Update $\{\Theta,\phi,\gamma_1,\ldots,\gamma_n,\boldsymbol{\mu}^{(1)},\ldots,\boldsymbol{\mu}^{(n)},\boldsymbol{\sigma}^{(1)},\ldots,\boldsymbol{\sigma}^{(n)}\}\leftarrow
    \mathrm{OptimizerStep}(\nabla \widehat{\mathcal{L}})$\;
    
    \If{the chosen mixing map is linear}{
        Normalize the columns of the mixing matrix\;
    }
}
Draw $R$ Monte Carlo samples from $q(\mathbf{Z})$, compute $\mathbf{S}^{(r)}=f_{\Theta}(\mathbf{Z}^{(r)})$, and estimate $\widehat{\boldsymbol{\mu}}_{\mathbf{S}}$ and $\widehat{\boldsymbol{\sigma}}_{\mathbf{S}}$ using \eqref{eq:mc-mean}--\eqref{eq:mc-std}\;
\caption{Training procedure of the structured adaptive latent source-wise diffusion model.}
\label{alg:salswdm}
\end{algorithm}

To further illustrate how the above components are integrated, the overall architecture is shown in Figure~\ref{fig:overall-architecture}. Each latent dimension is treated as one source branch. For the $k$-th branch, a trainable Gaussian initial variable $\mathbf{z}^{(k)}$ is first introduced and regularized toward a standard normal reference. It is then mapped through a source-wise reverse diffusion process to generate the recovered source trajectory $\mathbf{s}^{(k)}$. A source-specific GP prior is imposed on this trajectory to enforce temporal structure. After all branches are generated, the recovered source matrix $\mathbf{S}$ is passed through the mixing map $g_{\phi}$ to reconstruct the observed mixtures. In this way, source-wise initialization, diffusion generation, structured prior regularization, and mixture reconstruction are optimized jointly in one end-to-end framework.

\begin{figure*}[t]
    \centering
    \includegraphics[width=\textwidth]{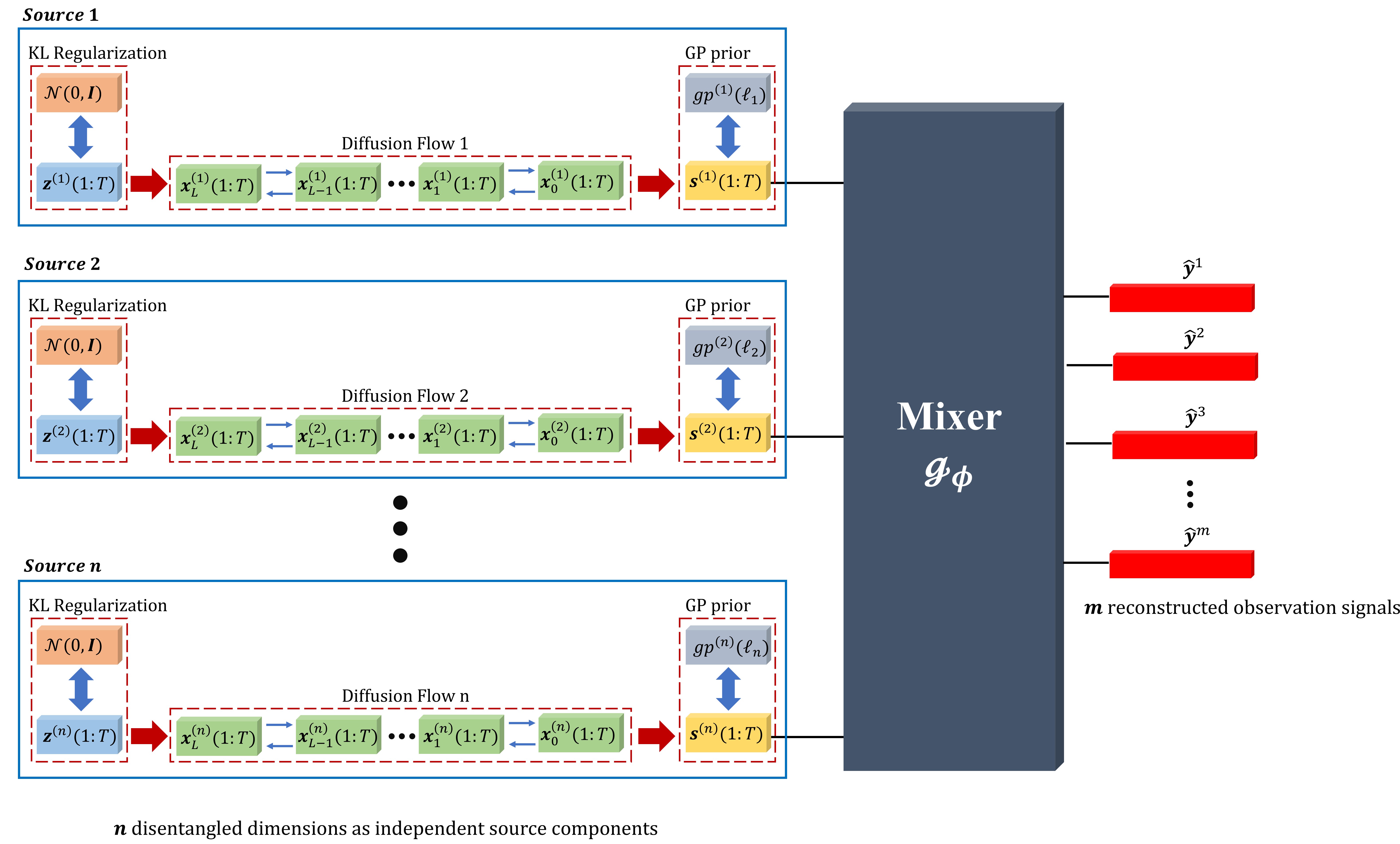}
    \caption{Overall architecture of the proposed StrADiff Framework.}
    \label{fig:overall-architecture}
\end{figure*}

\section{Experimental Study}

\subsection{Experimental setting}

To evaluate the proposed StrADiff framework, three artificial source signals with different temporal structures were used throughout the experiments. These three sources were designed to exhibit clearly different source-wise dynamics, so that the ability of the proposed source-wise adaptive diffusion mechanism and source-specific GP priors to separate heterogeneous temporal patterns could be examined more directly. Both linear and nonlinear mixing scenarios were considered. In the nonlinear case, the nonlinear mixing construction followed the same general experimental setting used in \cite{wei2026pdgmmvae}.

Unless otherwise stated, the number of sources was set to three, the reverse diffusion length was set to $L=20$, and the model was trained end-to-end for 10{,}000 epochs. For each source branch, the recovered source uncertainty was estimated by Monte Carlo sampling from the learned diffusion-start distribution followed by reverse diffusion. The reported source trajectories therefore correspond to Monte Carlo estimated means, while the shaded bands indicate the associated $95\%$ confidence intervals.

\subsection{Linear mixing results}

Figure~\ref{fig:exp-linear-separation} shows the final source recovery results in the linear mixing experiment. After permutation matching and sign correction, all three recovered sources are highly consistent with the corresponding true signals. The matched correlations are very close to one, indicating that the proposed framework can successfully recover the latent sources in the linear case. It is also worth noting that the estimated uncertainty bands are visually very narrow. This is because the Monte Carlo standard deviations at convergence are already very small, so the resulting $95\%$ confidence intervals are difficult to observe at the plotting scale. In other words, the final recovered sources are not only accurate but also highly concentrated.

\begin{figure*}[t]
    \centering
    \includegraphics[width=\textwidth]{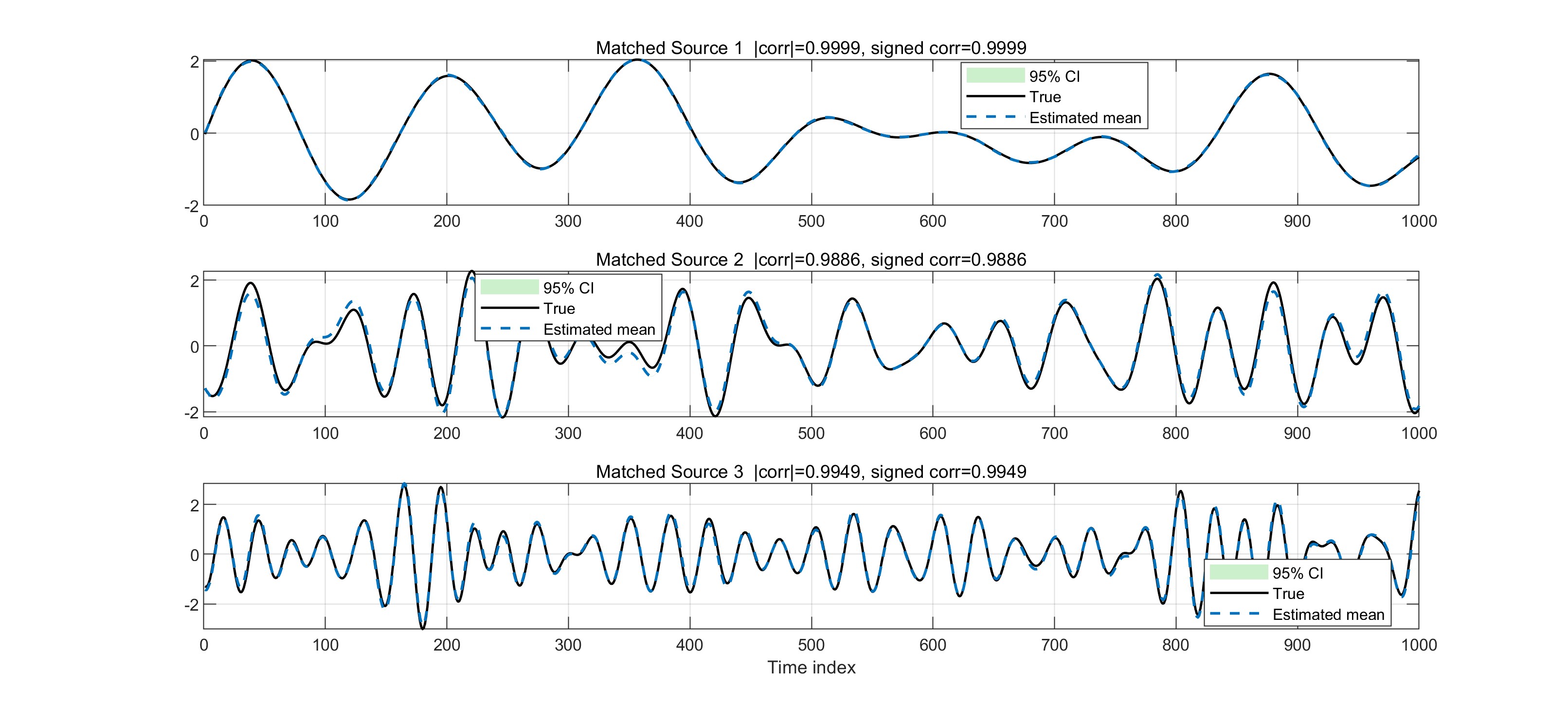}
    \caption{Linear mixing experiment: final matched source recovery results.}
    \label{fig:exp-linear-separation}
\end{figure*}

The training behavior for the same linear experiment is shown in Figure~\ref{fig:exp-linear-convergence}. The total loss, reconstruction term, GP term, diffusion term, and KL term all converge stably during optimization. The reconstruction MSE rapidly decreases to a very small level and remains low for the rest of training, indicating that the recovered sources and learned mixing map jointly explain the observations well. The correlation curves also rise quickly and remain close to one, showing that source separation emerges early and is then gradually refined during training.

The lower panels of Figure~\ref{fig:exp-linear-convergence} present the learned GP length-scales for the three source branches. Since the original time index $1{:}1{:}1000$ was normalized to $[0,1]$ before GP modeling, two versions are shown: the normalized length-scale and the rescaled length-scale in the original time unit. This presentation makes the learned temporal scales easier to interpret. The three source branches converge to different length-scales, which is consistent with the fact that the three artificial sources possess different temporal structures. 

\begin{figure*}[t]
    \centering
    \includegraphics[width=\textwidth]{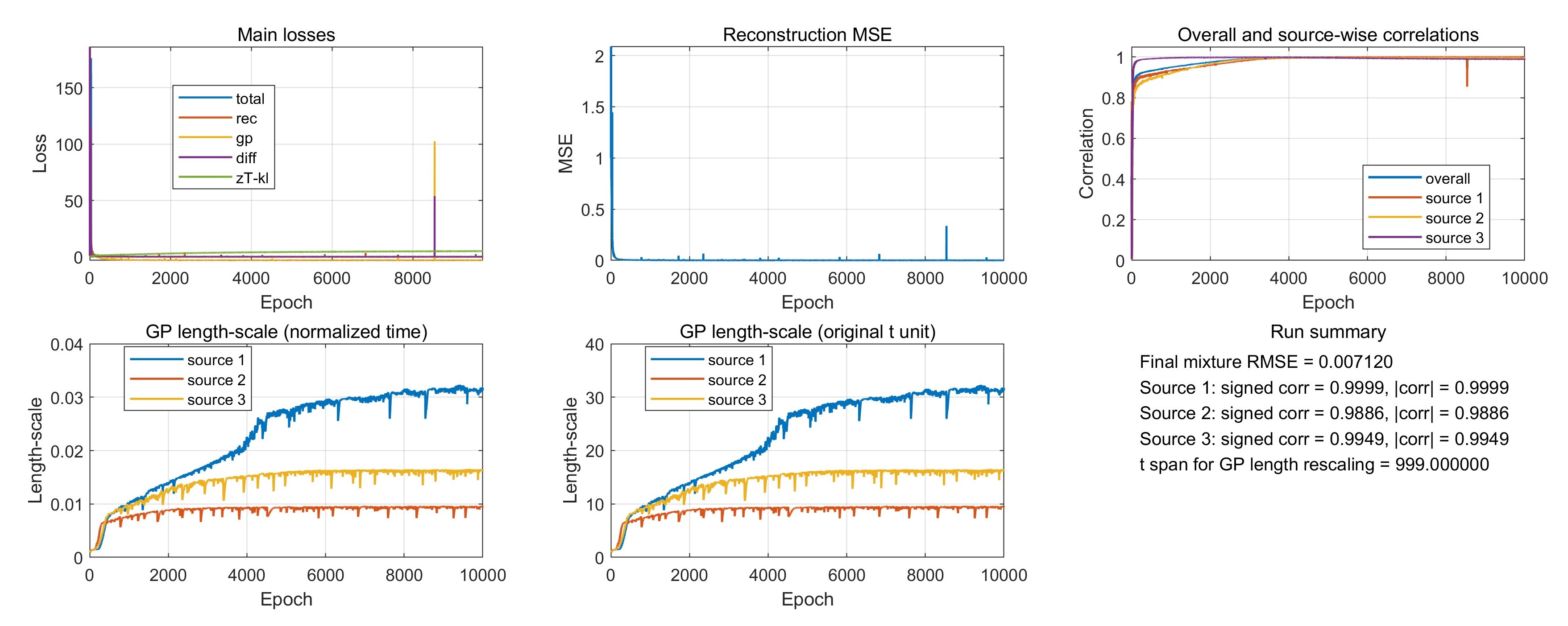}
    \caption{Linear mixing experiment: convergence of the main loss terms, reconstruction MSE, source-wise correlations, and learned GP length-scales.}
    \label{fig:exp-linear-convergence}
\end{figure*}

\subsection{Source-wise diffusion path analysis}

To further examine how the proposed reverse diffusion behaves inside each source branch, Figures~\ref{fig:exp-diff-start}--\ref{fig:exp-diff-final} visualize the diffusion trajectories at different training stages. In all three figures, each row corresponds to one source branch, and the columns show representative states along the reverse path, from the diffusion start state $\mathbf{x}_L$ to the final recovered state $\mathbf{x}_0$.

Figure~\ref{fig:exp-diff-start} shows the diffusion paths at the beginning of training. Since the reverse process has not yet learned the source structure, the trajectories at this stage still resemble Gaussian initial states. This is consistent with the intended construction of the model: the reverse diffusion starts from a source-wise Gaussian latent variable and only gradually learns how to map that initial random state toward a structured source trajectory.

\begin{figure*}[t]
    \centering
    \includegraphics[width=\textwidth]{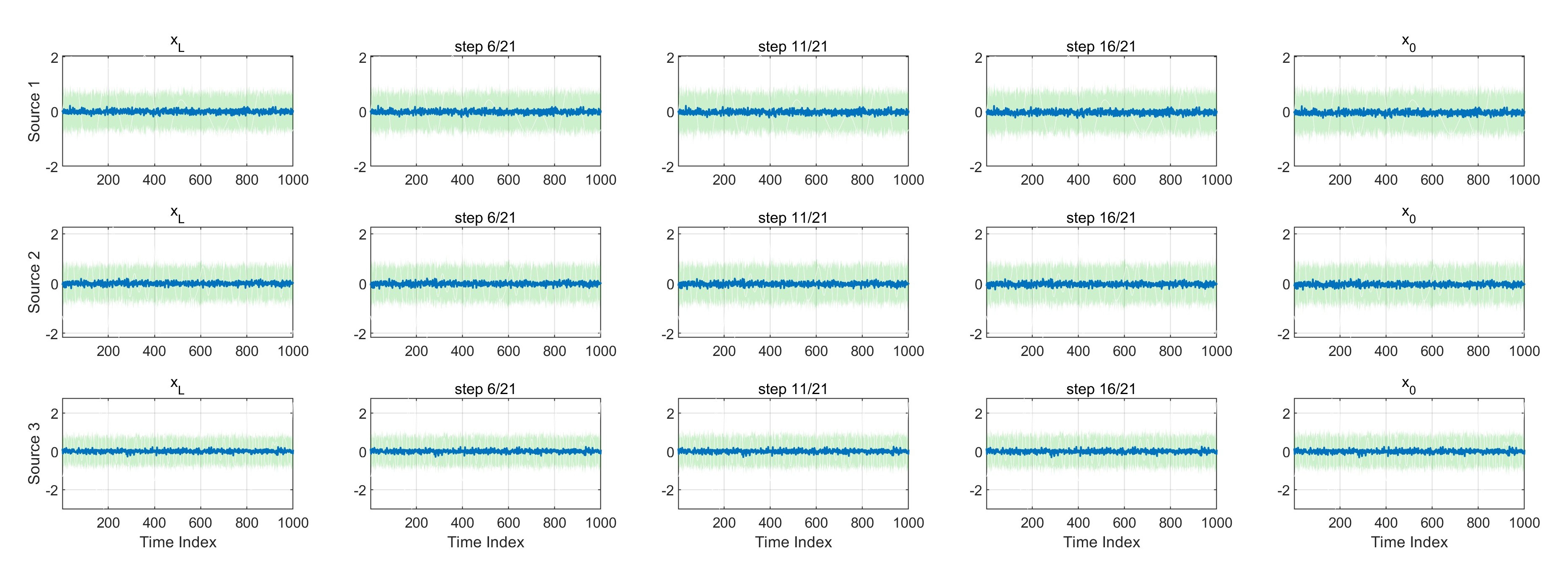}
    \caption{Reverse diffusion paths at the beginning of training.}
    \label{fig:exp-diff-start}
\end{figure*}

Figure~\ref{fig:exp-diff-mid} shows the corresponding diffusion paths around epoch 3000. By this stage, the reverse trajectories have already become much more structured and visibly closer to the target source shapes. At the same time, the Monte Carlo uncertainty has become very small, indicating that the learned reverse diffusion has already concentrated strongly around the recovered source trajectories. This suggests that the model does not merely memorize a final deterministic signal, but learns a stable source-wise generative path from $\mathbf{x}_L$ to $\mathbf{x}_0$.

\begin{figure*}[t]
    \centering
    \includegraphics[width=\textwidth]{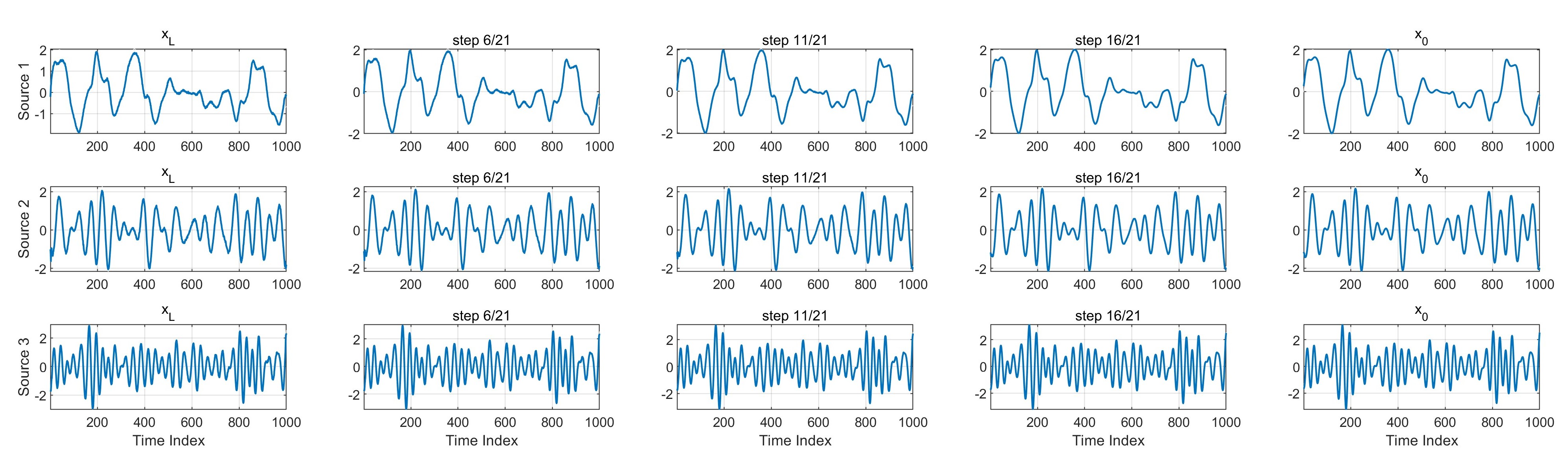}
    \caption{Reverse diffusion paths at an intermediate stage of training (epoch 3000).}
    \label{fig:exp-diff-mid}
\end{figure*}

Figure~\ref{fig:exp-diff-final} presents the diffusion paths at the final epoch. Compared with the initial stage, the reverse process now produces source trajectories that are highly structured from the beginning of the reverse path and remain stable across the sampled steps. The final state $\mathbf{x}_0$ matches the recovered sources well. Together, Figures~\ref{fig:exp-diff-start}--\ref{fig:exp-diff-final} illustrate the core mechanism of the proposed framework: each latent dimension is treated as a separate source branch, each branch owns its own adaptive reverse diffusion process, and the diffusion trajectories progressively organize themselves into distinct source-specific signal patterns during training.

\begin{figure*}[t]
    \centering
    \includegraphics[width=\textwidth]{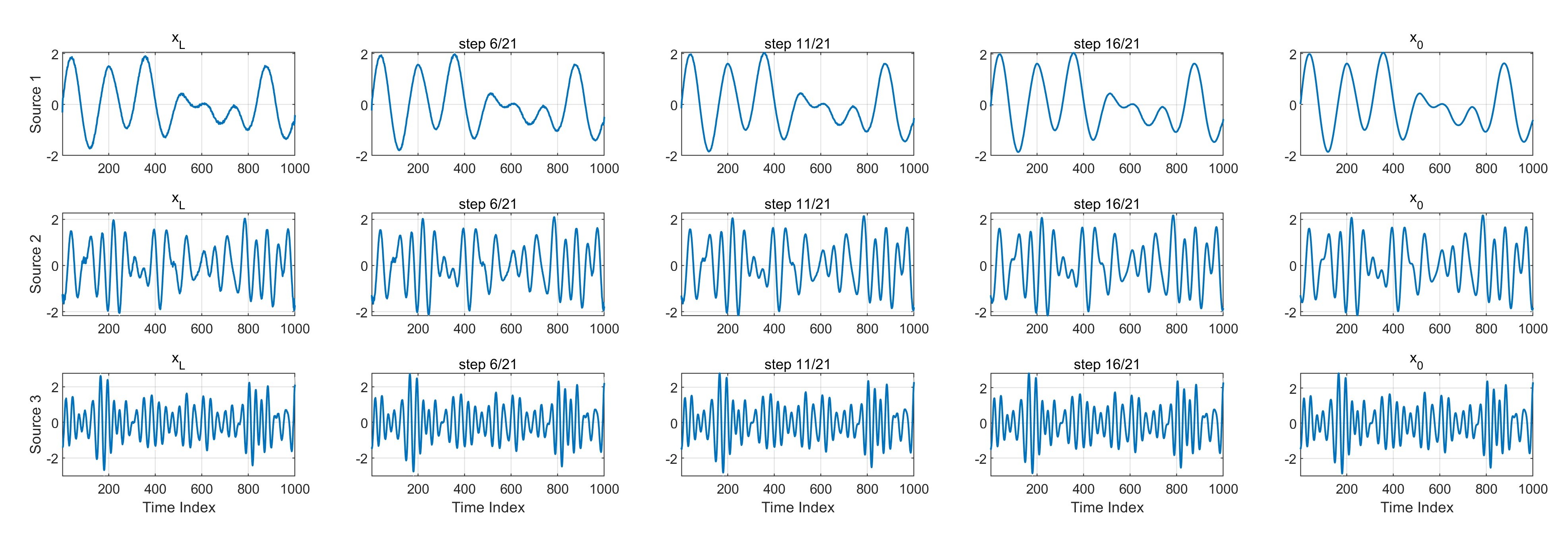}
    \caption{Reverse diffusion paths at the final epoch.}
    \label{fig:exp-diff-final}
\end{figure*}

\subsection{Nonlinear mixing results}

Figure~\ref{fig:exp-nonlinear-separation} shows the final recovered sources in the nonlinear mixing experiment. Compared with the linear case in Figure~\ref{fig:exp-linear-separation}, the nonlinear results remain satisfactory but are visibly less accurate. The recovered trajectories still follow the true source shapes well overall, but the matched correlations are lower than those in the linear experiment, and small local deviations can be observed more clearly in some segments.

\begin{figure*}[t]
    \centering
    \includegraphics[width=\textwidth]{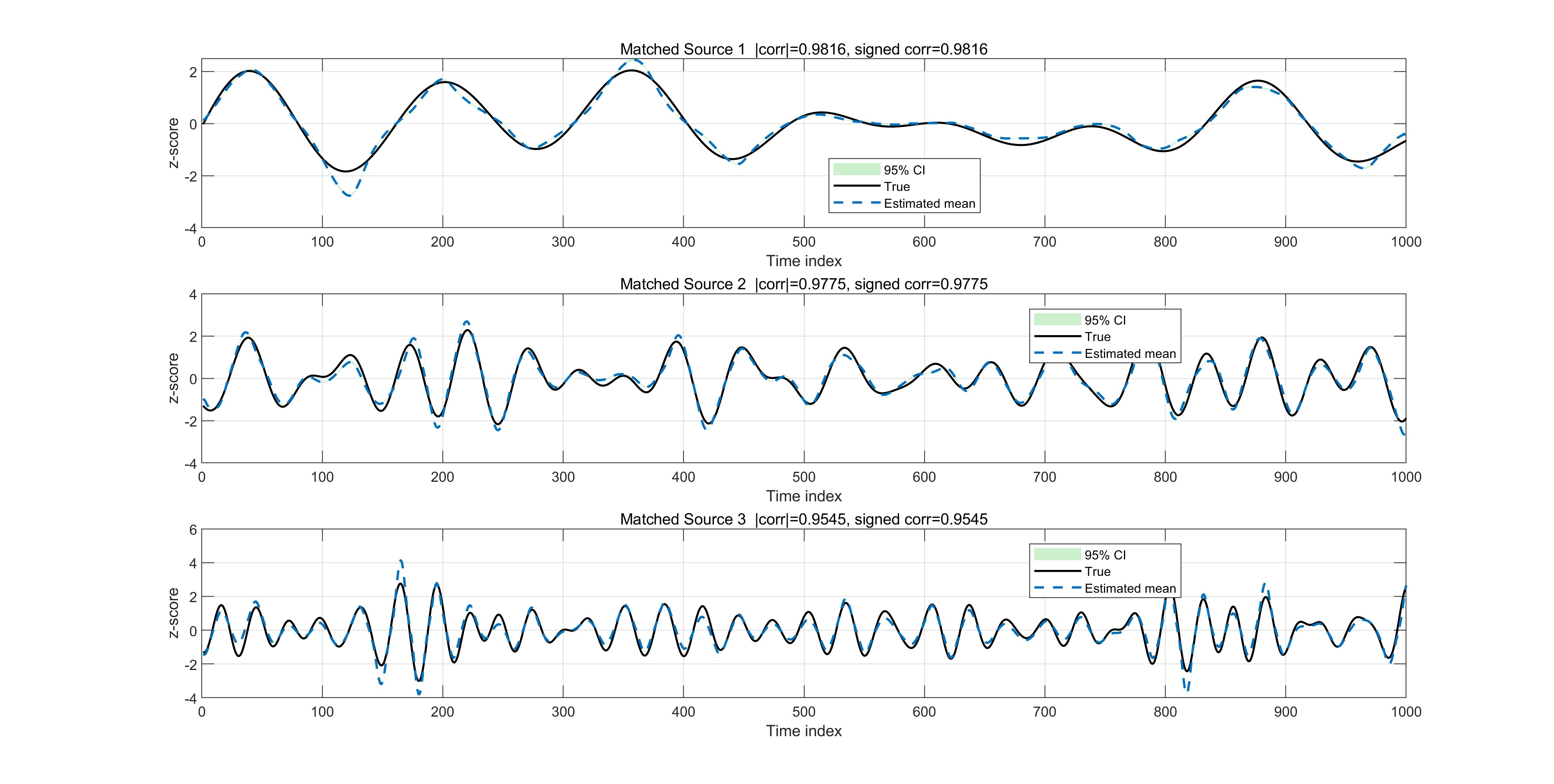}
    \caption{Nonlinear mixing experiment: final matched source recovery results.}
    \label{fig:exp-nonlinear-separation}
\end{figure*}

\subsection{Discussion}

Overall, the experiments support three main observations. First, in the linear setting, the proposed StrADiff framework achieves very strong separation performance, with highly accurate recovered sources, near-perfect correlations, and very small posterior uncertainty. Second, the learned GP length-scales differ across source branches, which is consistent with the use of heterogeneous temporal source priors and supports the source-wise modeling philosophy of the framework. Third, the diffusion-path visualizations provide direct evidence that the reverse diffusion mechanism is not merely an auxiliary loss term, but an active part of the source-generation process: as training proceeds, each source branch evolves from an approximately Gaussian initial state toward a structured and stable recovered source trajectory. In the nonlinear case, performance degrades moderately, but the framework still preserves meaningful source recovery ability.

\section{Conclusion}

This paper presented StrADiff, a structured source-wise adaptive diffusion framework for linear and nonlinear blind source separation. The method treats each latent dimension as a source branch and assigns to it an individual reverse diffusion mechanism, rather than relying on a single shared latent prior. In the present implementation, each branch is further regularized by its own adaptive Gaussian process prior, so that source estimation, source-wise latent regularization, diffusion learning, and mixture reconstruction are optimized jointly in one end-to-end unsupervised framework.

The results show that source-wise adaptive diffusion can provide a feasible structured generative route to blind source separation. In the linear case, the proposed framework achieves highly accurate source recovery, near-perfect matched correlations, and small Monte Carlo uncertainty. The learned GP length-scales differ across source branches, supporting the idea that different latent dimensions can adapt toward different temporal structures. The reverse diffusion path visualizations further suggest that the diffusion mechanism is not merely an auxiliary training loss, but an active source-generation process that progressively transforms source-wise Gaussian initial variables into structured latent trajectories. In the nonlinear case, the method still recovers meaningful source shapes, although the performance is less stable than in the linear setting.

Beyond the specific BSS experiments, the present framework also suggests a broader interpretation of ``sources''. In classical signal separation, a source usually denotes a physical or statistical signal component. In more complex latent-variable problems, however, the same source-wise formulation may be interpreted more generally: each branch can represent an independent, disentangled, or otherwise interpretable explanatory factor, provided that suitable structural assumptions or branch-specific priors are imposed. Under this view, BSS serves as a concrete and measurable testbed for examining whether source-wise diffusion branches can specialize into distinct latent roles during unsupervised training. The proposed formulation therefore has potential relevance not only to source separation, but also to structured representation learning, interpretable latent-variable modeling, and source-wise disentanglement.

\clearpage
\bibliography{ref}

\end{document}